\documentclass[10pt,conference]{IEEEtran}
\IEEEoverridecommandlockouts
\IEEEpubid{\makebox[\columnwidth]{ 978-1-6654-3540-6/22/\$31.00~\copyright~2022 IEEE \hfill} \hspace{\columnsep}\makebox[\columnwidth]{ }}
\usepackage{cite}
\usepackage{draftwatermark}
\usepackage{amsmath,amssymb,amsfonts}
\usepackage{graphicx}
\usepackage{wrapfig}
\usepackage{textcomp}
\usepackage{xcolor}
\usepackage{algorithmicx}
\usepackage{algcompatible}
\usepackage[ruled,vlined,linesnumbered]{algorithm2e}
\usepackage{amsthm}
\def\BibTeX{{\rm B\kern-.05em{\sc i\kern-.025em b}\kern-.08em
    T\kern-.1667em\lower.7ex\hbox{E}\kern-.125emX}}

\newtheorem{theorem}{Theorem}

\SetWatermarkText{Pre-print}

\begin{document}

\title{Prospect Theory-inspired Automated P2P Energy Trading with Q-learning-based Dynamic Pricing\\}

\author{\IEEEauthorblockN{Ashutosh Timilsina}
\IEEEauthorblockA{\textit{Department of Computer Science} \\
\textit{University of Kentucky}\\
Lexington, USA \\
ashutosh.timilsina@uky.edu}
\and
\IEEEauthorblockN{Simone Silvestri}
\IEEEauthorblockA{\textit{Department of Computer Science} \\
\textit{University of Kentucky}\\
Lexington, USA \\
simone.silvestri@uky.edu}
}
\onecolumn
\maketitle

\begin{abstract}
The widespread adoption of distributed energy resources, and the advent of smart grid technologies, have allowed traditionally passive power system users to become actively involved in energy trading. Recognizing the fact that the traditional centralized grid-driven energy markets offer minimal profitability to these users, recent research has shifted focus towards decentralized {\em peer-to-peer (P2P) energy} markets. In these markets, users trade energy with each other, with higher benefits than buying or selling to the grid. However, most researches in P2P energy trading largely overlook the user perception in the trading process, assuming constant availability, participation, and full compliance. As a result, these approaches may result in negative attitudes and reduced engagement over time. In this paper, we design an {\em automated} P2P energy market that takes user perception into account. We employ {\em prospect theory} to model the user perception and formulate an optimization framework to maximize the buyer's perception while matching demand and production. Given the non-linear and non-convex nature of the optimization problem, we propose Differential Evolution-based Algorithm for Trading Energy called $DEbATE$. Additionally, we introduce a risk-sensitive Q-learning algorithm, named Pricing mechanism with Q-learning and Risk-sensitivity ($PQR$), which learns the optimal price for sellers considering their perceived utility. Results based on real traces of energy consumption and production, as well as realistic prospect theory functions, show that our approach achieves a $26\%$ higher perceived value for buyers and generates $7\%$ more reward for sellers, compared to a recent state of the art approach. 


\end{abstract}



\section{Introduction}
Distributed Energy Resources (DER), such as rooftop solar and wind turbine, have seen widespread proliferation among  consumers in recent years\cite{IEA}. In addition, the advent of Smart Grid (SG) technologies, Advanced Metering Infrastructures (AMI), and home energy management systems, have added flexibility in energy generation/consumption for consumers. This, in turn, has allowed traditionally passive consumers to become actively involved in energy trading by sharing the excess energy generated at their premise to either grid or other buyers \cite{timilsina2021reinforcement,Parag2016prosumer_era}. These active consumers with energy production capabilities have been referred to as \textit{prosumers} \cite{Parag2016prosumer_era}, as a portmanteau of ``producers'' and ``consumers''. The role of prosumers in energy market has been recognized to some extent with the adoption of incentive schemes like \textit{Feed-in-Tariff} (FiT) mechanism \cite{tushar2020peer, tushar2018peer}. FiT allows prosumers to sell excess energy to the grid and buy from grid when required \cite{tushar2018peer}. However, existing  energy trading modalities offer limited  benefits to participating prosumers. This is due to the minimal prices at which energy is purchased by grid, as well as the low limits on the amount of energy that can be purchased \cite{tushar2020peer,tushar2018peer,Parag2016prosumer_era}.


\subsection{Literature Review and Motivation}
{\em Peer-to-peer (P2P) energy trading} is a recently proposed decentralized modality for energy sharing aiming at solving limitations of centralized techniques. This modality has been gaining significant traction recently \cite{tushar2018peer,tushar2020peer}. 
Specifically, P2P energy trading allows prosumers to trade energy among each other at a negotiated price with or without the involvement of the grid \cite{tushar2020peer}. It generates better monetary incentives for prosumers compared to existing mechanisms while also reducing their grid dependency \cite{tushar2018peer}. 
Additionally, increased local energy generation/consumption resulting from P2P trading leads to the minimization of overall system energy loss while providing an effective way to achieve demand side management\cite{Zhu2013SmartMicrogrids}. Benefits extend also to the grid operator, by providing 
savings in investments that would have been otherwise required to develop/maintain transmission infrastructure in a centralized power distribution architecture \cite{Parag2016prosumer_era,tushar2020peer}. 


P2P energy trading has received attention from the research community in recent years. The works in \cite{tushar2018transforming,tushar2019grid} present game theoretic approaches in a P2P setting, while a greedy rule-based P2P mechanism to assign energy among prosumers is proposed in \cite{azim2019feasibility} that includes mid-market pricing. Similarly, the physical aspects of P2P energy trading, such as power loss minimization and voltage regulation,  have been explored in \cite{nasimifar2019peer,paudel2020peer}.
These works, however, largely overlook the user behavior in designing their solutions. As established in \cite{Parag2016prosumer_era,tushar2018peer,timilsina2021reinforcement}, accommodating the user behavioral modeling in P2P energy trading ensures sustained participation from prosumers while incentivizing their contribution. In fact, the papers \cite{tushar2018transforming,tushar2019grid} consider prosumers to be actively involved and fully compliant with the system as rational decision-makers. First concern with this assumption is that the continuous online presence of participating prosumers with the system might not always be possible in real-world application. Secondly, research on user behavioral models and decision making \cite{gigerenzer2002bounded,agosto2002bounded} have found users to have \textit{bounded rationality}.
Therefore, requiring constant active participation  overwhelms the users and incentivizes non-rational decisions \cite{earl2016bounded}. In the worst case, it might even result in users opting to terminate their participation altogether \cite{agosto2002bounded,timilsina2021reinforcement}. 
In that light, the works in \cite{timilsina2021reinforcement,agate2020enabling} incorporates  bounded rationality and user preferences into P2P energy trading. However, it requires continuous human participation and assumes a simplistic linear model for user perception. Conversely, the authors of  \cite{tushar2018peer} limit their focus on coalition formation in game theoretic setting and do not explicitly consider user behavioral modeling. 


As a result, a prosumer-centric P2P energy trading model, that effectively incorporates the prosumers' decision-making behavior and their perceived loss/gain value from trading, is still lacking in the existing literature. Such a trading modality is expected to require minimal active participation from users while also ensuring their sustained involvement through the adoption of user behavioral modeling.
To this end, the framework of \textit{Prospect Theory} (PT) \cite{kahneman2013prospect} can be used to model the non-rational user behavior in the face of uncertain decision-making. It is often regarded as fairly accurate mathematical representation of human behavior 
\cite{kahneman2013prospect,el2016prospect, el2017managing}. 

Recently, there has been few efforts in integrating PT in energy related applications as well to capture the irrationality of users \cite{saad2016toward,el2017managing,wang2020prospect,yao2021distributed}. In relation to P2P energy trading, the authors in \cite{yao2021distributed} have proposed a PT-based distributed energy trading model to optimize trading decisions for prosumers in a competitive market. 
Although these papers model the user behavior in some ways, they require active participation from users and also assume that such behavior (e.g., the parameters of PT) is homogeneous for all the users. Social science studies, such as the one conducted in Italy \cite{contu2016modeling} to investigate the social acceptance of nuclear energy using an online survey, show that users exhibit significant heterogeneity in their preferences for the sources of energy. Neuroscience studies have also stressed the heterogeneity of humans in reference to PT parameters \cite{fox2009prospect}. Not capturing such heterogeneity provides little benefits in terms of user behavioral modeling. 



\subsection{Paper Contributions}

In this paper, we design a PT-based optimization framework for prosumer-centric P2P energy trading as shown in Fig. \ref{fig:system_overview_PT}. The framework aims at matching energy production and consumption (step $1$ in Fig. \ref{fig:system_overview_PT}) to maximize the perceived utility of individual buyers while taking into account the intrinsic heterogeneity of human perception. Given that the optimization problem is non-linear and non-convex,  we further devise a \textit{Differential Evolution}-based \cite{storn1997differential} metaheuristic algorithm called $DEbATE$ to solve the problem (\textit{energy allocation}, step $2$). In order to ensure minimal active participation of prosumers, we employ a Reinforcement Learning (RL) framework, called $PQR$, in tandem with $DEbATE$ to automate the pricing mechanism for sellers (\textit{pricing mechanism}, step $3$). In doing so, $PQR$ learns the  selling price for each sellers using a PT-based risk-sensitive Q-learning algorithm \cite{shen2014risk}. The output of the algorithms is then returned to the prosumers for executing the physical energy transactions (step $4$). 
Using real datasets for energy production and consumption, paired with recent survey data for PT perception modeling, results show that $DEbATE$ performs $25\%$ higher in buyer's perception and $7\%$ higher in seller's reward compared to state-of-the-art approach.



The major contributions of the paper are the following:
\begin{itemize}
    \item We develop a PT-inspired optimization framework for P2P energy trading;
    
    \item We design a metaheuristic algorithm $DEbATE$ to solve the non-linear energy allocation problem;
    
    \item We design dynamic pricing mechanism with $PQR$ algorithm using risk-sensitive Q-learning approach; 
    
    \item Experiments using real data show the superiority of proposed approach compared to the state-of-the-art;
\end{itemize}

 \begin{figure}[!thb]
 \includegraphics[width=.96\linewidth]{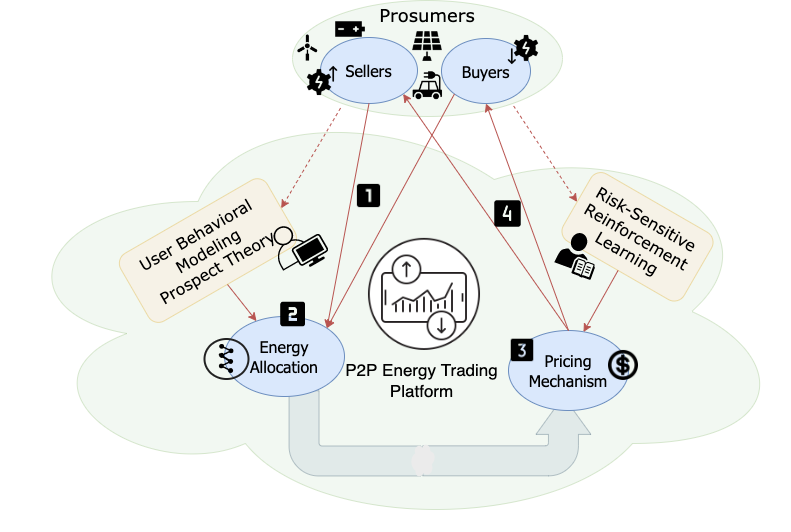}
 \caption{P2P Energy Trading System Overview.}\label{fig:system_overview_PT}
\end{figure}

\section{System Model and Problem Formulation}\label{chap:system_prob}
We consider a P2P energy trading system as shown in Fig. \ref{fig:system_overview_PT}. The system consists of prosumers that can exchange energy among each other through an existing distribution network. The grid serves as backup for  prosumers to either buy or sell energy, if the local energy trading is insufficient or not possible. Let $P$ be the set of all prosumers participating in the P2P energy market. We refer to  $B_t \subset P$ as the set of \textit{Buyers}, i.e. the set of prosumers that have higher self-consumption than generation at a timeslot $t$, and consumers without energy generation capabilities. Similarly, $S_t \subset P$ is the set of \textit{Sellers}, i.e., prosumers that have excess generation at a timeslot $t$. For simplicity of notation, we drop the subscript $t$ in the following. 


We model the perceived loss and gain of prosumers using the \textit{prospect theory} (PT) value function to capture user perception on gains and losses. Specifically, consider the excess energy generation of seller $i \in S$ be $r_i$ and demand of buyer $j \in B$ be $w_j$. Then, let $x_{ij} \in [0,1]$ represent the fraction of $w_j$ that a buyer $j$ is willing to buy from  seller $i$ at $\rho_{i}$ price per $kWh$ amount of energy. There is an {\em energy loss} during the physical energy transfer through wires \cite{Zhu2013SmartMicrogrids}, which depends on the wire-length between $i$ and $j$ and directly proportional to the amount of energy exchanged. The loss is modeled as a fraction $l_{ij} \in [0,1]$ of the energy exchanged.
Assume $\rho_{gs},\rho_{gb}$ be the energy selling and purchasing prices from the grid. We adopt a modified PT value function to model realistic user perception in an energy market  \cite{kahneman2013prospect}. The function quantifies  perceived utility of humans towards gain and loss based on degree of deviation from a reference point. Particularly, in our problem, it captures the difference of total actual buying cost $y_{j}$ from the buyer's desired total reference cost $\rho_{j}w_j$ where $\rho_j$ is the {\em reference price} of  buyer $j$ for purchasing energy. This utility function is formulated as   

\begin{equation} \label{eq:PT_val_buyers}
    v(y_j)= 
\begin{cases}
    k_{+,j}(\rho_{j}w_j - y_j)^{\zeta_{+,j}},               & y_j < \rho_{j}w_j\\
    -k_{-,j}(y_j - \rho_{j}w_j)^{\zeta_{-,j}},              & y_j \geq \rho_{j}w_j
\end{cases}
\end{equation}

where $k_{+,.}, k_{-.},\zeta_{+,.}, \zeta_{-,.}$ are the parameters that control the degree of loss-aversion and risk-sensitivity. These parameters are found to be highly heterogeneous and vary from person to person based on factors like gender and age group
\cite{balavz2013testing,rieger2017estimating}.  
$y_j$ is the total actual cost of buying energy for $j^{th}$ buyer s.t.
$$ y_j = \sum_{i \in S} \rho_{i}x_{ij}w_j + \rho_{gs}(1-\sum_i x_{ij})w_j$$
Note that, similar to the  PT value function in \cite{kahneman2013prospect}, the utility function in Eq. \eqref{eq:PT_val_buyers} is  concave in the gain domain (i.e. case $y_j < \rho_{j}w_j$) while convex in loss domain (i.e. case $y_j \geq \rho_{j}w_j$).

The problem of matching demand and production of heterogeneous prosumers is formalized as follows. 
\begin{subequations}\label{obj_func_pt}
\begin{align}
	{\text{maximize}}& &&f(y):\sum_{j \in B} v(y_j) \tag{\ref{obj_func_pt}}\\
 	\mbox{s.t.}& && \sum_{j \in B} (1+l_{ij}) x_{ij}w_j \leq r_i, &&&& \forall i\label{const2_pt}\\
	&&& \sum_{i \in S} x_{ij} \leq 1, &&&& \forall j\label{const3_pt}\\
	&&&  x_{ij} = 0 \text{, if } l_{ij}\geq l_{max},&&&& \forall i\label{const4_pt}\\
	&&& \rho_{gb} \leq \rho_{i},\rho_{j} \leq \rho_{gs},&&&& \forall i\label{const5_pt}\\
	&&& x_{ij} \in [0, 1], &&&&\forall i,j\label{const6_pt}
\end{align}
\end{subequations}

The problem maximizes the sum of perceived utility for buyers in Eq. \eqref{obj_func_pt}. Constraint in Eq. \eqref{const2_pt} prevents the problem from exceeding the amount of energy being sold by each sellers while incorporating the losses in electric lines. The constraint in Eq.  \eqref{const3_pt} ensures that the energy demand for each buyers is not exceeded, while constraint \eqref{const4_pt} limits the loss between sellers and buyers to be within the loss threshold $l_{max}$. Finally, the constraint \eqref{const5_pt} limits upper and lower bound for energy price to the selling and buying price of the grid.




It is to be noted that the problem in Eq. \eqref{obj_func_pt} is non-linear, non-convex 
optimization problem. Hence, we propose a heuristic based on Differential Evolution Algorithm (DEA) \cite{storn1997differential} described in the following section. Additionally, in the above problem, the selling price is considered as a fixed amount for a trading period. However, the reference price $\rho_j$ of buyer $j$ is a personal value which may under- or over-estimate the competitiveness of market. In order to maximize the sellers' perceived objectives through prospect theory, we resort to the risk-sensitive Q-learning algorithm \cite{shen2014risk}.

\begin{algorithm}[!hpbt]
\SetAlgoLined
\footnotesize
\caption{DEbATE}\label{alg:DEBATE}
\footnotesize
\SetKwInOut{Input}{Input}
\SetKwInOut{Output}{Output}
\Input{set of buyers $B$, sellers $S$, fitness function $f(.)$, max iterations $G_{max}$, population size $NP$, crossover probability $CR$, differential weight $F$}
\Output{best identified feasible solution $\mathbf{x^*}$}
Update set of buyers $B$ and sellers $S$, $count = 0$\; 
Generate initial population $\mathcal{X} = \{\mathbf{x_k} |\mbox{ } k = 1, \dots, NP\}$\; 
\While{$count  < G_{max}$}{
    \For{each $\mathbf{x_k} \in \mathcal{X}$}{
        Choose $3$ different vectors $\{\mathbf{x_a}, \mathbf{x_b}, \mathbf{x_c}\}\in \mathcal{X}$  at random and $R \sim U(1,|S| \times |B|)$\;
        Create mutated solution $\mathbf{\Bar{x}_k} = \mathbf{x_k}$\; 
        \tcc{\textbf{Mutation and Crossover}}
        \For{each $i \in |S|$, $j \in |B|$}{
            Select $u \sim U(0,1)$ \;
            \uIf{$u < CR || (i \times j) == R$}{
                $\Bar{x}_{ij}^{(k)} = x_{ij}^{(a)} + F \times (x_{ij}^{(b)} - x_{ij}^{(c)})$\;
                $\Bar x_{ij}^{(k)} = \min(1,\max(0,\Bar x_{ij}^{(k)}))$
            }
            }
        \tcc{\textbf{Check Constraints}}
        $\forall i,j$, \lIf{$l_{ij}\geq l_{max}$}{$\Bar{x}_{ij}=0$}
        $\forall i$, \lIf{$\sum_j (1+l_{ij})\Bar x_{ij}w_j > r_i$}{$\Bar{x}_{ij} = \frac{\Bar{x}_{ij}r_{i}} {\sum_{\hat{j}} \Bar (1+l_{i\hat{j}})\Bar{x}_{i\hat{j}}w_{\hat{j}}} $}
        $\forall j$, \lIf{$\sum_i \Bar x_{ij} > 1$}{ $\Bar{x}_{ij} = \frac{\Bar{x}_{ij}} {\sum_{\hat{i}} \Bar x_{\hat{i}j}} $}
        \tcc{\textbf{Compare fitness}}
        \lIf{$f(\mathbf{\Bar{x}_k}) > f(\mathbf{x_k})$}{$\mathcal{X} = (\mathcal{X} \setminus \{ \mathbf{x_k} \}) \cup \{\mathbf{\Bar{x}_k}\}$}
    }
    count =  count++\;
}
\tcc{\textbf{Find the best solution to execute trading}} 
Let $\mathbf{x^*} = \arg \max\limits_{\mathbf{x_k} \in \mathcal{X}} f(\mathbf{x_k})$\;


Execute transactions for each prosumers to $\mathbf{x^*}$ \;
\end{algorithm}

\section{The DEbATE and PQR Heuristics} \label{chap_fut_heuristic}
In this section, we describe the \textit{Differential Evolution-based Algorithm for Trading Energy (DEbATE)} (Alg. \ref{alg:DEBATE}), designed for the problem presented in Section \ref{chap:system_prob}, and the  \textit{Pricing mechanism with Q-learning and Risk-sensitivity (PQR)}, designed to dynamically adjust the sellers' prices. 

\subsection{DEbATE}
$DEbATE$ is executed at each trading period (e.g., 12 hours) to solve the non-linear optimization problem in Eq. \eqref{obj_func_pt}. It uses differential evolution to determine an optimal amount of energy to be traded between prosumers that maximizes the perceived utility of  buyers.
\textit{DEbATE} initially updates the list of buyers ($B$) and sellers ($S$) based on the expected production and consumption for current trading period. These can be  predicted accurately with recent approaches \cite{kong2017short,casella2022dissecting}. 
The differential evolution-based optimization begins on line $2$ where an {\em initial population} $\mathcal{X}$ is generated with population size of $NP$. An element $\mathbf{x_k} \in \mathcal{X}$, with  $k=1,2,\dots,NP$ is  a {\em candidate solution} vector of variables $x_{ij}$ representing the amount of energy to be traded between each seller $i$ and buyer $j$ . These variables correspond to the decision variables of our optimization problem.

The $while-$loop (line $3-19$) is the differential evolution loop that aims at finding solution to the non-linear optimization problem with Eq. \eqref{obj_func_pt} as the fitness function. The loop is executed for $G_{max}$ iterations. 
At each iteration, for each candidate solution $\mathbf{x}_k \in \mathcal{X}$, the algorithm creates a {\em mutated solution} $\mathbf{\bar{x}_k}$. Initially, $\mathbf{\bar{x}_k} = \mathbf{x_k}$.
The mutated solution is subsequently updated through mutation and crossover with $3$ random candidates $\mathbf{x}_a, \mathbf{x}_b, \mathbf{x}_c \in \mathcal{X}$ (line $5$). A value $R \in [1,|S|\times |B|]$ is selected at random. $R$ will be used in the following $for-$loop to ensure a minimum mutation. The for loop in line $7$ iterates over the components (dimensions in evolutionary terms) of $\mathbf{\bar{x}_k}$. 
During each iteration, a value $u \in [0,1]$ is sampled at random as mutation probability (line $8$). Subsequently, a mutation occurs for the component $ij$ of $\mathbf{\bar{x}_k}$ with crossover probability $CR$ (line $9$). The mutation occurs irrespective of the probability  if $(i \times j) = R$ (to ensure at least one minimum mutation). A mutation is executed by combining the  corresponding component of $\mathbf{x_a}$, $\mathbf{x_b}$, and $\mathbf{x_c}$ with  the differential weight parameter $F \in [0,2]$ as in line $10$.  
The mutated component $\mathbf{\bar{x}_{ij}^{(k)}}$ is clipped to ensure that it falls within $[0,1]$ as minimum and maximum threshold to satisfy constraint Eq. \eqref{const6_pt} in line $11$ of the algorithm. 


After the mutated solution is finalized, it is checked, and adjusted if needed, to meet the constraints in Eqs. \eqref{const2_pt}-\eqref{const4_pt}  of the optimization problem. Specifically, line $13$ ensures that no exchange occurs (i.e., $\mathbf{\bar{x}_{ij}^{(k)}} = 0$) between users having a loss higher than $l_{max}$. Lines $14-15$ ensure that the production of a seller and the demand of each buyer are not exceeded, respectively.
Finally, in line $16$, the fitness function $f(.)$ of the mutated solution $\mathbf{\Bar{x}_k}$ is compared against the original candidate solution $\mathbf{x_k}$. If $f(\mathbf{\Bar{x}_k}) > f(\mathbf{{x}_k}) $, then $\mathbf{\Bar{x}_k}$ replaces $\mathbf{{x}_k}$ in the set of candidate solutions $\mathcal{X}$.
At the end of the while loop, $DEbATE$ selects the best solution $\mathbf{{x}^*}$ in $\mathcal{X}$ (line $20)$ and executes the transactions accordingly (line $21$).
In the following theorem \ref{Theo:complexity}, we show that the $DEbATE$ has polynomial complexity and hence, computationally efficient.

\begin{theorem}\label{Theo:complexity}
The complexity of the $DEbATE$ algorithm is $O(G_{max} \times NP \times |S||B|)$. 
\end{theorem}
\begin{proof}
The complexity is dominated by the  $while$ loop (lines $3-19$), which is executed $G_{max}$ times. Within this loop, the $for-$loop (lines $4-17$) does $|\mathcal{X}| =  NP$ total iterations. In each iteration, the inner $for-$loop (lines $7-12$) iterates over the sets $S$ and $B$, and only contains constant operations. Similarly, checking the constraints (lines $13-15$) requires to iterate over the same sets. Finally, calculating the function $f(.)$ (line $16$) has cost $|B|$. Overall, the complexity is $O(G_{max} \times NP \times (|S||B| + 3|S||B| + |B|)) = O(G_{max} \times NP \times |S||B|)$
\end{proof}

\begin{algorithm}[!hpbt]
\SetAlgoLined
\footnotesize
\caption{PQR}\label{alg:PQRs}
\footnotesize
 \tcc{\textbf{Pricing with Risk-sensitive Q-learning}} 
 Collect transaction information for each prosumers from $DEbATE$ (Alg. \ref{alg:DEBATE}) for current timestep $t$\;
 \For{each $i \in S$}{
 Select an action, $a \in \{+\delta,-\delta,0\}$ based on exploration and exploitation \;
 $s=\rho_i;s_{new} = s+a; R_i = (\rho_i+a) \sum\limits_{j \in B} x_{ij}$\;
 Update $Q(s,a)$ as in Eq. \eqref{eq:PT_Q_update}\;
 $\rho_i = s_{new}$\;
 Send information on updated price $\rho_i$ to seller $i$\;
 }
\end{algorithm}

\subsection{PQR}
After determining the solution to the energy allocation problem in $DEbATE$, the selling price for sellers is then updated through the $PQR$ algorithm. In order to learn the optimal selling price dynamically over time, we model the sellers as independent learning agents. Note that, to preserve the privacy and avoid the conflict between prosumers, these agents do not have access to information about other sellers or buyers. The state space in the Q-learning formulation consists of the prices between the grid buying ($\rho_{gb}$) and selling ($\rho_{gs}$), discretized by a step size, $\delta$, i.e., $\rho_i \in \{\rho_{gb},\ \rho_{gb}+\delta,\ \rho_{gb}+2\delta,\ ...,\ \rho_{gb} + \big(\frac{\rho_{gs}-\rho_{gb}}{\delta}-1\big)\delta,\ \rho_{gs}\}.$

The action space consists of a price increasing action, price decreasing action, and no change action, i.e. $a \in \{+\delta,-\delta,0\}$, where $\delta$ is the amount by which price is increased or decreased. Seller $i$ reward function is the total revenue generated at the current trading period i.e. $R_i = (\rho_{i}+a)\sum_{j \in B} x_{ij}w_j$. For updating Q-values, we modify the approach proposed in \cite{shen2014risk} by considering the following Q-learning update rule that includes the PT-based perceived utility of sellers.

\begin{equation} \label{eq:PT_Q_update}
    Q^{(new)}(s,a) = Q^{(old)}(s,a) +\alpha v(y_i)
\end{equation}

\begin{equation} \label{eq:PT_val_sellers}
    v(y_i)= 
\begin{cases}
    k_{+,i}(y_i)^{\zeta_{+,i}},               & y_i > 0\\
    -k_{-,i}(-y_i)^{\zeta_{-,i}},              & y_i \leq 0
\end{cases}
\end{equation}

where, $y_i = R_i + \gamma \max_a Q(s_{new},a) - Q(s,a)$ is the Temporal Difference (TD) error of $i^{th}$ seller for current iteration, and $v(y_i)$ is transformation of TD error to capture each seller's personalized perceived utility on loss and gain. $\alpha$ refers to the learning rate for updating Q-values in Eq. \eqref{eq:PT_Q_update}.
The action is selected based on an \textit{$\epsilon$-greedy} exploration-exploitation strategy \cite{sutton2018reinforcement}. Specifically, $\epsilon$ refers to the probability of exploration and it is initially set to $1$. It is then decreased over time using an \textit{$\epsilon-$decay} value, as the system learns the optimal policy.
Based on the selected action, the new selling price, reward, and Q-value are updated as per Eqs. \eqref{eq:PT_Q_update} and \eqref{eq:PT_val_sellers}. Updated selling price is then sent to the respective seller $i$ for next trading period. 



The system runs both $DEbATE$ and $PQR$ sequentially at every trading period. Input of $DEbATE$ is updated based on the prices computed by $PQR$. $PQR$ then takes as input the reward from executing energy transactions by $DEbATE$.

\section{Experimental Results}

 \subsection{Experimental Setup}
 
  \begin{figure}[htbp]
      \centering
      \includegraphics[width=0.5\linewidth]{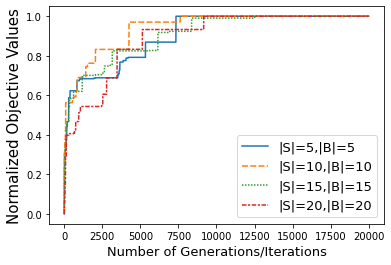}
      \caption{Normalized objective value vs. number of iterations.}
      \label{fig:de_obj_val}
  \end{figure}

 \begin{figure*}[!th]
\minipage{0.33\linewidth}
  \includegraphics[width=.99\linewidth]{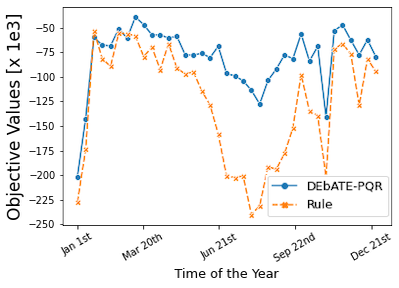}%
  \caption{Buyers' perceived values.}\label{fig:obj_val_PT}
\endminipage\hfill
\minipage{0.33\linewidth}%
 \includegraphics[width=.99\linewidth]{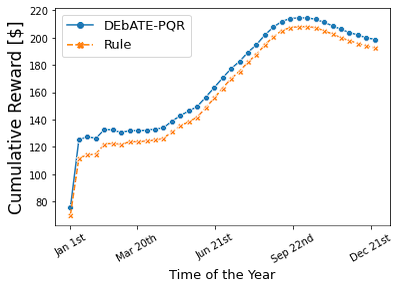}
 \caption{Sellers' cumulative reward.}\label{fig:cum_reward_PT}
\endminipage\hfill
\minipage{0.33\linewidth}%
 \includegraphics[width=0.99\linewidth]{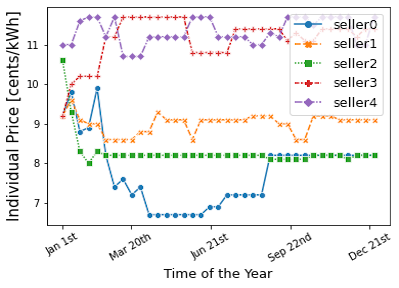}
 \caption{Individual prices.}\label{fig:ind_price_PT}
\endminipage
\end{figure*}

In this section, we evaluate the performance of \textit{DEbATE} and \textit{PQR}, hereafter jointly referred as $DEbATE-PQR$, against a recent state-of-the-art approach referred to as \textit{Rule} \cite{azim2019feasibility}. $Rule$ allocates energy using a greedy heuristic that assigns cheapest sellers to buyers based on their registration order in the system, while final price of each transaction follows mid-market pricing, i.e., mid value of seller's and buyer's asking price. We consider a system with  $40$ prosumers, split evenly as buyers and sellers. This is considered  a representative number of prosumers in a microgrid or set of houses supplied by a single distribution transformer. We use a realistic dataset for buyers' energy consumption obtained from \cite{Energy_dataset}. Similarly, we consider sellers equipped with $4kW$ rooftop solar located in Lexington, Kentucky, USA.
The energy generated is estimated using NREL's PVWatts Calculator \cite{Solar_dataset} given the solar irradiance in Lexington and size of solar panels. 
Losses are assigned uniformly at random from set $\{1\%, 2\%, 3\%, 4\%\}$ and maximum loss threshold $L_{max} = 2.5\%$. 

We assume that prosumers complete a survey before joining the system to estimate their individual prospect theory parameters, similar to \cite{rieger2017estimating,balavz2013testing,fox2009prospect}, and use realistic prospect theory parameters determined by them. 
Specifically, we sample the risk-averting parameter for gains $(\zeta_+) \in [0.60,0.88]$, the risk-seeking parameter for losses $(\zeta_-) \in [0.52, 1.0]$,  the loss-aversion parameters for gain and loss $(k_+),(k_-) \in [2.10,2.61]$ for each individual prosumers. The grid energy buying price is set to $\rho_{gb} = \$ 0.06$ and the selling price to $\rho_{gs} = \$ 0.12$. The reference price for each sellers is initially randomly sampled from range $[0.09, 0.12]$. It is then updated using $PQR$ at each iteration. The reference price for each buyer is selected in the range $[0.06,0.10]$ and considered static for the duration of experiments, which is $365$ days. The parameters for $PQR$ algorithm are set as follows: learning rate $\alpha = 10^{-4}$, step size for discretizing state space $\delta =\$ 0.001$, and $\epsilon-$decay $=0.965$.   
 

 \subsection{Results}
 We consider several experimental scenarios and performance metrics, as discussed in the following.
 


\textbf{Experimental Scenario 1:}
We first run experiments to study the convergence of \textit{DEbATE}. We considered different system size by scaling the number of sellers and buyers.   Fig. \ref{fig:de_obj_val} shows the normalized objective value as a function of the number of iterations using a population size $NP=20$. The plot averaged over 10 runs shows that $10,000$ iterations are sufficient for the algorithm to converge in the considered settings. As a result, in the following scenarios we set  $G_{max} = 10,000$ and the population size $NP=20$.

\textbf{Experimental Scenario 2:} In the second experimental scenario we study the performance of the considered approaches over time.  Two performance metrics are considered, namely the buyers' objective value  and the sellers' cumulative reward. These are represented in Figs. \ref{fig:obj_val_PT} and \ref{fig:cum_reward_PT}, respectively, with a moving average of $10$ days. In this experiments we consider $15$ buyers and $15$ sellers.  The benefits of $DEbATE-PQR$ over $Rule$ are more prominent from April through October, when the energy demand and production is higher. The greedy nature of $Rule$ penalizes the quality of the resulting matching, significantly reducing the buyers' perceived value.  Note that, the buyers' objective values are negative because they are paying higher prices than their reference purchase price. Therefore, transactions are seen as loss from a prospect theory perspective. Nevertheless, our approach optimizes the energy assignment to maximize the buyers perceived value. Additionally, our approach is able to generate higher rewards than $Rule$ by dynamically learning the prices for sellers through the $PQR$ algorithm.  The the sellers' reward decreases after mid-september for both the approaches due to the reduced energy production during winter.


   \begin{figure*}[tbhp]
\minipage{0.5\linewidth}
  \includegraphics[width=.99\linewidth]{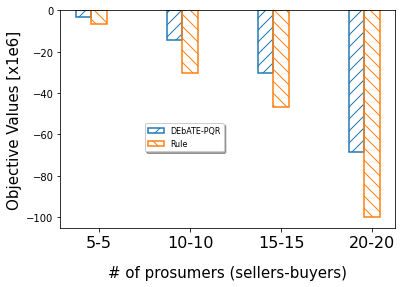}%
  \caption{Obj. values for buyer vs. network size.}\label{fig:obj_size_PT}
\endminipage\hfill
\minipage{0.5\linewidth}%
 \includegraphics[width=.99\linewidth]{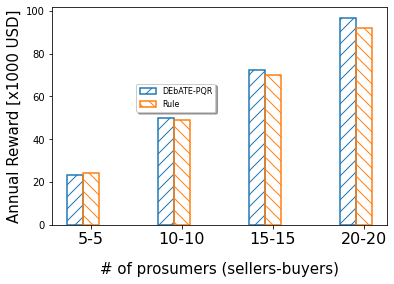}
 \caption{Total rewards for sellers vs. network size.}\label{fig:tot_reward_PT}
\endminipage
\end{figure*}
 
We further study the performance over time by considering the evolution of average and individual sellers' prices. We consider a smaller system of $5$ sellers and $5$ buyers for ease of representation of the results.  
Fig. \ref{fig:ind_price_PT} shows the individual prices. $DEbATE-PQR$ is able to learn and adjust the price over time to improve the buyers' perceived value considering their competitiveness. The competitiveness is a function of a buyer's reference price, their production, and their location in the system (e.g., loss w.r.t. sellers). As a result, our approach is able to improve the perception of both buyers and sellers while ensuring the competitiveness of the market.



\textbf{Experimental Scenario 3:} In this scenario we test the scalability with respect to the system size. Specifically, we increase the system  proportionately from $5$ sellers and $5$ buyers to $20$ sellers and $20$ buyers.
Figs. \ref{fig:obj_size_PT}-\ref{fig:tot_reward_PT} show the buyers' total perceived value and the sellers' reward, respectively, over a year. 
By considering the loss-averse and risk-seeking PT-value functions, $DEbATE-PQR$ achieves an increasing advantage as the system size increases compared to  $Rule$, for both sellers and buyers. As a numerical example,  $DEbATE-PQR$ achieves as much as $26\%$ increase in buyers' perceived value while ensuring $7\%$ profit improvement for sellers.

\section{Concluding Remarks} 
In this paper, 
we bring together the concept of perceived utility from behavioral economics and reinforcement learning into the P2P energy trading scene. Unlike existing literature, 
we propose an automated and dynamic P2P energy trading problem that maximizes the perceived value for buyers while simultaneously learning the optimal selling price. Given the non-linear and non-convex nature of the problem, we propose a novel differential evolution-based metaheuristic algorithm, called $DEbATE$. $DEbATE$ is paired with a prospect theory enhanced Q-learning algorithm, called $PQR$, to adjust the selling price over time. Results show the advantages of the proposed approaches with respect to a state of the art solution using real energy consumption and production data.

\section*{Acknowledgment} 
This work is 
supported by the NSF grant EPCN-1936131 and NSF CAREER grant CPS-1943035.

\appendices

\bibliographystyle{IEEEtran}
\bibliography{IEEEabrv,bibliography}

\end{document}